\documentclass{article}

\usepackage[english]{babel}
\usepackage[letterpaper,top=2cm,bottom=2cm,left=3cm,right=3cm,marginparwidth=1.75cm]{geometry}
\usepackage{amsmath}
\usepackage{amssymb}
\usepackage{amsthm}
\usepackage{mathtools}
\usepackage{dsfont}
\usepackage{amsfonts}
\usepackage{graphicx}
\usepackage{natbib}
\usepackage{thmtools}
\usepackage{hyperref}
\usepackage{nameref,cleveref}
\usepackage{enumitem}
\usepackage{tikz}
\usepackage{cancel}

\usepackage{color,xcolor}
\usepackage[normalem]{ulem}
\usepackage{mdframed}

\newcommand{\remove}[1]{}

\newcommand{\AAA}{\mathcal{A}}
\newcommand{\BBB}{\mathcal{B}}
\newcommand{\DDD}{\mathcal{D}}
\newcommand{\CCC}{\mathcal{C}}

\newcommand{\R}{\mathbb{R}}
\newcommand{\N}{\mathbb{N}}

\newcommand{\Sim}{{\rm Sim}}

\newcommand{\Lap}{{\rm Lap}}
\newcommand{\SideInfo}{{\rm SideInfo}}

\newcommand{\eps}{\varepsilon}

\newcommand{\1}{\mathds{1}}

\newcommand{\gray}[1]{\textcolor{gray}{#1}}

\theoremstyle{definition}
\newtheorem{theorem}{Theorem}[section]
\newtheorem{definition}[theorem]{Definition}

\newtheorem{corollary}[theorem]{Corollary}
\newtheorem{remark}[theorem]{Remark}
\newtheorem{example}[theorem]{Example}

\DeclareMathOperator{\error}{error}

\crefname{definition}{Definition}{Definitions}
\Crefname{definition}{Definition}{Definitions}

\crefname{type}{Type}{Types}
\Crefname{type}{Type}{Types}

\crefname{target}{Target}{Targets}
\Crefname{target}{Target}{Targets}

\crefname{attack}{Attack}{Attacks}
\Crefname{attack}{Attack}{Attacks}

\title{Protecting the Undeleted in Machine Unlearning}
\author{
Aloni Cohen\thanks{University of Chicago. {\tt <aloni@uchicago.edu>}.} 
\and 
Refael Kohen\thanks{Tel Aviv University. {\tt <refael.kohen@gmail.com>}. Partially supported by the Israel Science Foundation (grant 1419/24) and the Blavatnik Family Foundation.}
\and 
Kobbi Nissim\thanks{Dept.\ of Computer Science, Georgetown University. {\tt <kobbi.nissim@georgetown.edu>}. Work partly supported by NSF award \#2217678.} 
\and 
Uri Stemmer\thanks{Tel Aviv University. {\tt <u@uri.co.il>}. Partially supported by the Israel Science Foundation (grant 1419/24) and the Blavatnik Family Foundation.}}

\begin{document}

\date{February 18, 2026}

\maketitle

\begin{abstract}
Machine unlearning aims to remove specific data points from a trained model, often striving to emulate ``perfect retraining'', i.e., producing the model that would have been obtained had the deleted data never been included. We demonstrate that this approach, and security definitions that enable it, carry significant privacy risks for the {\em remaining} (undeleted) data points. We present a reconstruction attack showing that for certain tasks, which can be computed securely without deletions, a mechanism adhering to perfect retraining allows an adversary controlling merely $\omega(1)$ data points to reconstruct almost the entire dataset merely by issuing deletion requests. We survey existing definitions for machine unlearning, showing they are either susceptible to such attacks or too restrictive to support basic functionalities like exact summation. To address this problem, we propose a new security definition that specifically safeguards undeleted data against leakage caused by the deletion of other points. We show that our definition permits several essential functionalities, such as bulletin boards, summations,  and statistical learning.\\
\end{abstract}

\section{Introduction}

Sensitive personal data underlies the training of machine learning models across many domains, necessitating compliance with legal data protection standards such as the EU General Data Protection Regulation (GDPR) and the California Consumer Privacy Act (CCPA). 
These legal standards provide individuals with a number of protections and rights, including the right to request that their specific information be deleted, a.k.a.\ the Right to be Forgotten (RTBF). This legal provision gave rise to the field of {\em Machine Unlearning}, a branch of machine learning focused on enabling the removal of elements from the training data without needing to rebuild models from scratch.

We now make our discussion more formal and survey several influential definitions for machine unlearning.
Consider a data curator $\CCC$ that executes an algorithm $\AAA$ on an initial input dataset $D_0=(x_1,\dots,x_n)$ to obtain an outcome $h_0$ (say a hypothesis or a  statistic). 
Throughout the execution, at every time $i$, the curator $\CCC$ receives a {\em deletion request} of the form ``delete point $x_j$'' and needs to respond with an updated outcome $h_i$ that should reflect the state of the dataset after this $i$th deletion request. 
Starting from the work of \citet{cao2015towards}, {\em machine unlearning} was defined as a computational problem, where the ideal solution is for the curator to retrain the learning algorithm $\AAA$ from scratch in every step on the remaining dataset. 
The goal of the machine unlearning algorithm is to emulate this ideal solution (perfectly or approximately) faster, i.e., without having to spend the work needed to re-run $\AAA$ for every deletion request. 
For example, if $\AAA$ returns the exact sum $h_0=x_1+\cdots+x_n$, deleting $x_j$ could be done via a single subtraction $h_1 = h_0-x_j$ (instead of summing the $n$ numbers again). This way, deletion is performed in constant rather than linear time.

Following \citet{cao2015towards}, several other definitions were suggested, including by \citet{ginart2019making,guo2019certified,sekhari2021remember,gupta2021adaptive}, with the same goal of mimicking/approximating the behavior of retraining from scratch in every iteration (``the ideal solution'') while saving on runtime. 
These definitions differ in the approximation metrics they consider and in their environmental settings (such as whether or not the deletion requests are adaptive to past outputs). 
The focus of our work is not on runtime; therefore, for our purposes, all these definitions can be grouped together as {\em definitions that are satisfied by perfect retraining}. See Section~\ref{sec:otherRelated} for a survey of prior definitions.

\paragraph{Perfect retraining and privacy.} While the right to delete one's information is rooted in privacy standards such as the GDPR, the relationship between data deletion and privacy is not straightforward.
In particular, it is well-known that definitions satisfied by perfect retraining fail under a simple ``differencing attack'' where an adversary learns information about the deleted item by examining the output hypothesis before and after a deletion (e.g., in the exact sum example, if an attacker sees the outcome before and after Alice deletes her input, then the attacker learns Alice's input exactly). 

Nevertheless, one may argue that such compromise of Alice's personal data is acceptable as long as Alice is aware of her privacy risks and still wants to exercise control over her data and delete it from the model. 
However, as we will see, there are cases where the deletion of Alice's data can compromise the privacy of {\em other} users who did not ask for their data to be deleted. 
More generally, we will show that there are tasks for which 
any machine unlearning mechanism, under any definition satisfied by perfect retraining, allows an attacker
controlling a small number of parties to reconstruct essentially all the training data. This, we believe, should be prevented.

\paragraph{Privacy-oriented definitions of deletion.}
There are several privacy-oriented definitions in the literature that are {\em not} satisfied by perfect retraining (and thus might not be susceptible to our attacks). Perhaps the first example of such a (folklore) definition would be one that requires the use of differential privacy (DP) with small enough privacy parameters to capture group privacy for all the deleted users.\footnote{That is, at the outset of the execution, the curator releases some DP computation on the initial dataset (with very small privacy parameters) and then simply ignores all deletion requests.} 
However, to the best of our knowledge, all existing definitions that are not satisfied by perfect retraining are too restrictive in the sense that they disallow natural and important functionalities, such as a simple sum or a bulletin board. Again, we refer the reader to Section~\ref{sec:otherRelated} for a survey of prior definitions.

\subsection{Our contributions}\label{sec:ourCont}

\paragraph{Perfect retraining allows exposure of undeleted elements.}
We present settings in which a coalition of a small number of parties can reconstruct the information of other parties by issuing data deletion requests. In the simplest setting, one party, Alice, who asks for her data to be deleted, causes the information of another party, who did not ask for their data to be deleted, to be exposed. In this scenario:  
(i) Alice could not learn the other party's data had she not deleted her data, and 
(ii) Alice can learn the other party's data once the deletion is performed.
We believe that enabling such behavior is unacceptable; thus, we argue that definitions that aim to mimic or approximate perfect retraining should be used with caution. 

A simple example where the deletion of one party's data causes another's to be revealed is in median computation. Given $n$ data points $x_1\leq x_2\leq \cdots \leq x_n$ their median is $x_{\frac{n+1}{2}}$ if $n$ is odd and $x_{\frac{n}{2}}$ otherwise. Observe that the initial median computation reveals one data point. Assume $n$ is odd and that Alice's entry is greater than the median. When Alice asks to delete her data, a value smaller than the median is exposed -- the input provided by some another user. (See Appendix~\ref{sec:kMeansClustering} for another example).

The median example is perhaps somewhat controversial because the exact median itself is a highly non-private computation even without any deletions. Thus, one might argue that the privacy violation generated in this context is more because of the functionality itself rather than because of the deletions.\footnote{Taking this argument to the limit, consider the functionality $f$ where given a dataset $D$ we have that $f(D)=\emptyset$ if $|D|\geq n$, and $f(D)=D$ otherwise. Now, if the initial dataset is of size $n$, then by deleting her data Alice causes the function $f$ to leak all of the (remaining) dataset.}
However, in Section~\ref{sec:attacks} we present a task %
which can be performed with differential privacy (without supporting deletions). Nevertheless, an attacker controlling only a small number of parties, asking them to delete their data one by one, would force any perfect-retraining algorithm to leak almost all of its dataset.

\begin{theorem}[informal]\label{thm:attackInformal}
There exists a task $T$ such that
\begin{enumerate}
    \item[(1)] There exists an $\left(o(1),o(1)\right)$-differentially private algorithm $\AAA$ such that for any dataset $D$ we have that $\AAA(D)$ solves $T$ w.r.t.\ $D$ with high probability.

    \item[(2)]
    Perfect retraining w.r.t.\ $T$ allows an attacker controlling merely $\omega(1)$ points to reconstruct almost all of the dataset. 
    Specifically, there is an attacker $\BBB$ and a set of points $B=\{b_1,\dots,b_r\}$ of size $|B|=\omega(1)$ such that the following holds. ($B$ is the set of points controlled by $\BBB$.) For any dataset $D$ containing $B$, and any algorithm $\AAA$ for $T$, given $\left(\AAA(D),\AAA(D\setminus\{b_1\}),\AAA(D\setminus\{b_1,b_2\}),\dots,\AAA(D\setminus B)\right)$, the attacker $\BBB$ can reconstruct almost all of $D$.
\end{enumerate}
\end{theorem}

By item (1), without deletions, the task $T$ can be solved while guaranteeing strong privacy protections for the input dataset. Nevertheless, by item (2), an attacker that controls merely $\omega(1)$ points and can request their deletion from the dataset can reconstruct almost all of the dataset (reconstruct $\approx n$ points by deleting $\omega(1)$ points). As we mentioned, many existing definitions for machine unlearning are satisfied by perfect retraining. Theorem~\ref{thm:attackInformal} shows that these definitions do not rule out mechanisms that leak their entire dataset after just $\omega(1)$ deletion requests.

\begin{remark}
Item~(1) of Theorem~\ref{thm:attackInformal} is stated with  $(\eps,\delta)$-DP for $\delta=o(1)$. This is not standard, but still prevents the attacker from reconstructing, say, 1\% of the rows in $D$. The formal details are given in Appendix~\ref{sec:weakAttacker}. In particular, this can be restated with arbitrary  $(\eps,\delta)$ parameters, at the cost requiring the attacker described in Item~(2) to control $\omega\left(\frac{\log(1/\delta)}{\eps^2}\right)$ points. 
\end{remark}

\paragraph{A new definition for deletion safety.} Motivated by these vulnerabilities, we propose a new framework for machine unlearning. Unlike prior definitions that focus on removing the influence of the deleted point to mimic a world where it never existed, our definition focuses on protecting the privacy of the remaining, {\em undeleted} elements. Informally, we require that an attacker who observes an initial model and submits a sequence of deletion requests gains no additional information about the remaining dataset beyond what is inherent in the initial model and the deleted values themselves. We formalize this via a simulation-based definition, and provide variants handling non-adaptive and adaptive adversaries, as well as a relaxation allowing for bounded leakage.

\paragraph{Examples satisfying our new definition.} We show that our new definition is satisfied by stateless algorithms. This allows us to show that the exact sum and the bulletin board functionalities, which were disallowed by previous rigorous definitions, are compatible with our definition.
Furthermore, we provide examples for mechanisms, which were studied by prior work under the umbrella of perfect retraining, that are nevertheless compatible with our new definition. In other words, even though our work highlights security vulnerabilities of perfect retraining as a security notion, this does not necessarily mean that existing algorithms in the literature are bad; indeed, some of them could be justified using our new security notion. 

\subsection{A survey of prior definitions}\label{sec:otherRelated}

\paragraph{Definitions that are satisfied by perfect retraining.} 
Our attacks apply to all definitions which are satisfied by perfect retraining. 
This includes all the definitions in the literature in which perfect retraining is the ``stated goal''/``optimum''~\cite{cao2015towards,ginart2019making,guo2019certified,sekhari2021remember,gupta2021adaptive}. It also includes definitions where retraining from scratch is sufficient even if not the stated goal.

One such definition, called {\em deletion-as-control}, was introduced by \citet{cohen2023control}. 
Recall that we consider a data curator $\CCC$ that holds an initial dataset $D$ and sequentially receives deletion requests for elements from $D$ (and updates its responses accordingly). 
Informally, deletion-as-control requires that the internal state of the curator $\CCC$ in the real-world (where, say, Alice contributes her information and then asks to delete it) would be distributed similarly to its state in a ``hypothetical world'' where Alice does not contribute her data in the first place, but all other parties (except the curator) behave as if she did. Note that this definition is satisfied by perfect retraining because in perfect retraining the curator need not have a state besides the dataset, which is identical in the two executions (after the deletion).

Another notable definition by \citet{godin2021deletion} requires that the internal state of the curator $\CCC$ after Alice deletes her data should be simulatable from the view of all other parties (except Alice and the curator). 
Note that perfect retraining satisfies this definition, because the curator does not have a state besides the dataset, which consists of information jointly held by the parties excluding Alice. 

\paragraph{Definitions that are too restrictive.}
Not all definitions are satisfied by perfect retraining; against these our attacks do not apply.
As we argue next, these definitions can be too restrictive in some settings.
In particular, all the definitions we discuss next rule out simple and important functionalities including exact sums and bulletin board~\cite{cohen2023control}.
As a concrete example to have in mind, a bulletin board should allow users to publicly post messages and to later remove them.

Perhaps the most influential definition in this context, called {\em deletion-as-compliance}, was introduced by \citet{garg2020formalizing}. Suppose that Alice contributes and later deletes her data by interacting with the curator $\CCC$.
Informally, the requirement in the definition of \citet{garg2020formalizing} is that both the internal state of the curator $\CCC$ as well as the {\em view} of all other parties (except the curator and Alice) should be distributed similarly to the way they would have been distributed in an ideal world in which Alice never contributed her data in the first place. This means, in particular, that the view of these other parties (capturing everything they see throughout the execution) cannot be affected by Alice. 
Bulletin boards and exact sums are impossible: the view of other parties is inherently affected by Alice's contribution.

This issue was addressed by \citet{gao2022deletion} who  presented a weaker variant of the definition of \citet{garg2020formalizing} with the goal of being more permissive. Informally, the definition of \citet{gao2022deletion} is as follows. Suppose that Alice contributes {\em two} records and later deletes  {\em one} of them (by interacting with the curator $\CCC$). The requirement is that all other parties together should not be able to distinguish (based on their view of the execution) whether Alice deleted the first or the second record. This definition is indeed more permissive and, in particular, permits the {\em fact that a deletion has occurred} to be visible (but the identify of the deleted record needs to be kept hidden). However, this definition still rules out the  exact sum and the bulletin board functionalities for the same reasons as above.

Finally, \citet{chourasia2023forget} studied privacy risks that may arise when the adversary that selects the sequence of deletion requests is {\em adaptive} to the prior outputs of the curator.\footnote{This work extends the work of \citet{gupta2021adaptive} that we mentioned above. \citet{gupta2021adaptive} also studied adaptivity in machine unlearning, tackling the concern that algorithms that aim to mimic perfect retraining efficiently might lose their guarantees when the sequence of deletion is selected adaptively.} %
Their work was motivated by observation that if the sequence of deletions is adaptive, then this sequence itself %
might continue to reveal information about a record even after this record was deleted from the system. 
Informally, their definition requires that for every adaptive adversary $Q$ (that adaptively specifies deletion requests), every initial dataset $D$, every record $y\in D$, and every time step $i$, the output of the curator after receiving $i-1$ deletion requests from $Q$ and then a request to delete $y$ is simulatable from $D\setminus\{y\}$.\footnote{More specifically, their definition requires that for every adaptive adversary $Q$ and every time step $i$ there exists a simulator $\pi^{Q}_i$ such that the following holds. For every initial dataset $D$ and every record $y\in D$, the output distribution of the curator at time $i$, after receiving $i-1$ deletion requests (chosen adaptively by $Q$) and then receiving a request to delete $y$, is similar to the outcome distribution of the simulator $\pi^Q_i\left(D\setminus\{y\}\right)$.} 
Notice that the simulator of \citet{chourasia2023forget} is given all the data except for a single record, including every undeleted record. This if very much at odds with the ethos of our goal of privacy for those who remain.

Still, as \citet{chourasia2023forget} prove, the fact that this must hold for any {\em adaptive} adversary $Q$ forces the curator to provide some form of privacy protections to the records in $D$. 
(Roughly, the functionality can be non-private for at most a single time step.) Thus it disallows simple functionalities such as the exact sum or the  bulletin board. 
To see this for the bulletin board functionality, consider an adaptive $Q$ that behaves as follows: 
If Alice's public message mentions Bob, then $Q$ will request to delete Alice then Bob. Otherwise, $Q$ will request to delete Alice then Charlie. Since the simulator is executed without Alice's record, it cannot simulate which record to delete next.

Note that the definition of \citet{chourasia2023forget} argues about every single time step separately. In Section~\ref{sec:GenDifAttacks}, we show that such definitions are vulnerable to a type of attack we call a {\em generalized differencing attack}.

\section{Attacks}\label{sec:attacks}

In this section, we present several attacks in which an attacker controlling some of the records in a dataset can use the ability to delete these records in order to learn information about {\em other} records in the dataset.

\subsection{Reconstruction attacks}
Let us reconsider the median example, with the goal of releasing an {\em approximate} median rather than the exact median. Specifically, consider an algorithm that given a dataset $D$ returns a differentially private (DP) approximation for the median. This breaks the attack described in Section~\ref{sec:ourCont}, as it prevents the attacker from directly learning information about the remaining points from their DP median. Still, by deleting more and more points the attacker gets to see more and more releases (approximate medians), which together might allow it to infer information about the remaining points. If we were to ensure that {\em all releases together} satisfy DP, as in the {\em continual observation model} of \citet{DworkNPR10}, then this would provably prevent such an attack. However, this can come at a high price in terms of utility. Specifically, \citet{JainRSS23} presented problems that can be solved efficiently with differential privacy in the {\em single shot} setting (where only one output is released), but not in the {\em continual} setting, where the input dataset evolves over time and DP must be preserved across all outputs together. %

We now briefly describe a variant of one of the attacks of \citet{JainRSS23}, tailored to our setting. Essentially, the only conceptual modification from the attacks presented by \citet{JainRSS23} is that we highlight the fact that it is not just that DP cannot be maintained in this continual setting, but it is in fact the privacy of ``undeleted'' points that get compromised.\footnote{This also holds in the attacks of \citet{JainRSS23}, but not explicit as this was not their focus; they aimed to show that the error must be large as a function of the size of the evolving dataset (which includes both the deleted and undeleted elements in our context).} This will not yet result in the attack described in Theorem~\ref{thm:attackInformal} as the number of users that the attacker needs to control will be relatively large.
Consider the following problem: Given a multiset $T\subseteq[n]$, the goal is to return an approximation to
$$
{\rm CountMod}(T)\triangleq\sum_{i\in [n]} \1\Big\{ {\rm count}_T (i) \; {\rm mod}\;3=2\Big\},
$$
where ${\rm count}_T (i)$ is the number of occurrences of $i$ in $T$.
Note that ${\rm CountMod}$ is a function of sensitivity 1, and hence could be approximated with DP (in the single shot setting) with error roughly $1/\eps$.

Now consider a dataset (multiset) $T=D\cup S\subseteq[n]$, where $D$ is a random subset of $[n]$ of size $n/2$, and where $S$ is the multiset of size $3n^2$ containing $3n$ copies of every $i\in[n]$. We consider an attacker that ``controls'' $S$ (in the sense that it can request the deletion of points from $S$) and aims to learn the dataset $D$. Now observe that by controlling $S$, the attacker can force any mechanism for ${\rm CountMod}$ to provide it with answers to {\em counting queries} w.r.t.\ $D$. In more detail, a {\em counting query} is defined via a predicate $q:[n]\rightarrow\{0,1\}$, and the value of such a query on $D$ is $q(D)\triangleq\sum_{i\in D}q(i)$. Given such a counting query $q$, the attacker acts as follows:
\begin{enumerate}
    \item\label{step:queryInit} For every $i\in[n]$ such that $q(i)=1$, request the deletion of 2 copies of $i$ from $S$ (and thus from $T$).
    \item Obtain a release $a_q\approx {\rm CountMod}(T)$. Note that ${\rm CountMod}(T)=q(D)$, because, by Step~\ref{step:queryInit}, for every $i\in[n]$ such that $q(i)=0$ we have that ${\rm count}_T(i) \; {\rm mod}\; 3\neq2$, and for every $i\in[n]$ such that $q(i)=1$ we have that ${\rm count}_T(i) \; {\rm mod}\; 3=2$ if and only if $i\in D$.
    \item For every $i\in[n]$ such that $q(i)=1$, request the deletion of one copy of $i$ from $S$ (and thus from $T$). Note that this ``resets'' the modifications we did in Step~\ref{step:queryInit} in the sense that for every $i$, its number of occurrences in $S$ is a multiple of 3.
\end{enumerate}
This allows the attacker to obtain answers to (at least) $n$ arbitrary counting queries w.r.t.\ $D$. This is known to allow the attacker to reconstruct $D$ unless the approximation error is $\Omega(\sqrt{n})$ by the results of \cite{dwork2008new}.

\paragraph{Reducing the number of deletions.} 
In the last attack, the attacker controlled $O(n^2)$ points and were able to reconstruct $\approx n$ ``undeleted'' points. By slightly modifying the definition of the ${\rm CountMod}$ problem, it can be shown that controlling $O(n)$ points also suffices. The idea is that in the construction of \cite{dwork2008new}, the $n$ counting queries are fixed in advance (Fourier queries) and so we could ``hard-code'' them into the definition of the problem. Let these queries be denoted as $q_1,\dots,q_n$, and let $k$ be a parameter (think about $k\approx\frac{1}{\eps}$).  
Now let $T$ be a dataset (multiset) containing elements from $[n]\cup\{\star\}$, and consider the following task. Let ${\rm Star}(T)$ denote the number of $\star$'s in $T$. The goal is to approximate the value of $q_{i}$ on $D=\{x\in T : x\neq \star\}$ for an integer $i$ satisfying 
\begin{equation}\label{eq:validIndex}
\left\lfloor \frac{{\rm Star}(T)-k}{3k} \right\rceil \leq i \leq \left\lfloor \frac{{\rm Star}(T)+k}{3k} \right\rceil.
\end{equation}
Note that this definition of the problem does not require the algorithm to compute ${\rm Star}(T)$ exactly; it suffices to estimate it up to an error of $k$, which can be done privately provided that $k\approx\frac{1}{\eps}$. Also note that when ${\rm Star}(T)$ is a multiply of $3k$, say ${\rm Star}(T)=3kr$, then there is exactly one ``allowed'' index $i$, namely $i=r$.

Now consider a dataset (multiset) $T=D\cup S\subseteq[n]\cup\{\star\}$, where $D$ is a random subset of $[n]$ of size $n/2$ and where $S$ is the multiset containing $3kn$ copies of $\star$. Now observe that by controlling $S$ and deleting its elements one by one, the attacker gets (approximate) answers to all of the queries $q_n,q_{n-1},\dots,q_1$ on the dataset $D$, and could hence reconstruct $D$ unless the error is $\Omega(\sqrt{n})$. So now the attacker controls $O(n)$ points and reconstructs $O(n)$ other (``undeleted'') points.

\paragraph{Further reducing the number of deletions.} 

We can further reduce the number of points that the attacker controls (and deletes) by modifying the task such that the algorithm is required to release approximate answers to {\em multiple} queries together, rather than just one. More specifically, as before, let $q_1,\dots,q_n$ denote a (fixed) set of $n$ counting queries, and let $k,t\in\N$ be parameters where $t$ divides $n$. Given a dataset (multiset) $T\subseteq[n]\cup\{\star\}$, the new goal is to approximate the value of each of $q_{(i-1)t+1},...,q_{it}$ on $D=\{x\in T : x\neq \star\}$ for an integer $i$ satisfying Inequality~(\ref{eq:validIndex}). 
As before, a valid value for $i$ can be computed in a differentially private manner, provided that $k\gtrsim\frac{1}{\eps}$. In addition, as long as $t=o(n)$, a (single shot) DP algorithm can release answers to $t$ queries with error $o(\sqrt{n})$. Thus, the problem can be solved with privacy in the single shot setting (without deletions). However, with deletions, it now suffices for the attacker to control merely $\omega(1)$ points. Specifically, consider a dataset (multiset) $T=D\cup S\subseteq[n]\cup\{\star\}$, where $D$ is a random subset of $[n]$ of size $n/2$ and where $S$ is the multiset containing $s_0=|S| =  3kn/t=\omega(1)$ copies of $\star$. Now observe that by controlling $S$ and deleting its elements one by one, the attacker gets (approximate) answers to all of the queries $q_n,q_{n-1},\dots,q_1$ on the dataset $D$, and could hence reconstruct $D$ unless the error is $\Omega(\sqrt{n})$. The upside here is that the attacker needs to control very few points ($\omega(1)$ points) and by deleting them the attacker can reconstruct a dataset of size $n$. This results in Theorem~\ref{thm:attackInformal}. The formal details are given in Appendix~\ref{sec:weakAttacker}.

\subsection{Generalized differencing attacks}\label{sec:GenDifAttacks}
As we mentioned, definitions that aim to mimic/approximate the behavior of perfect retraining are susceptible to ``differencing attacks'', such as the simple counting example, where seeing two consecutive releases (before and after a deletion) allows the attacker to learn the deleted point. We now present simple generalizations of this attack that extend to other types of definitions:
\begin{enumerate}
    \item Consider a learning algorithm $\AAA$ that takes a dataset $D$ containing $n$ numbers, represented in binary (say using $d$ bits). The algorithm ignores the data and returns a random binary vector of length $d$, denoted as $z$. Next, upon receiving a request to delete a point $x$, the algorithm deletes $x$ from the data and returns $z\oplus x$. This, of course, allows the same ``differencing attack''. However, note that now the marginal distribution of every single release given by the algorithm is simply uniform over $d$-bit vectors, independent of the data. This shows that definitions that only argue about every single time step individually, like the definition of \citet{chourasia2023forget} mentioned above, can be vulnerable to differencing attacks even if they require that the distribution of every release is independent of the data.\footnote{For this we need to assume that either the curator is allowed to hold a state (the vector $z$), or that the deletion algorithm receives an input both the point to be deleted (the point $x$) and the previous release (the vector $z$ in our case).}

    \item Somewhat artificially, a similar attack could be leveraged to reveal information about ``undeleted'' points. The attack begins the same, with the initial release being a random $d$-bit vector $z$. Next, upon receiving a request to delete a point $x$, the algorithm deletes $x$ from the data, samples a random point $y$ from the remaining data points, and returns $z\oplus y$.
\end{enumerate}

\section{Our new definition\raisebox{1px}{\small (}s\raisebox{1px}{\small )}}

Consider a setting where a curator wishes to compute a functionality $f$ on an initial dataset $D$ while supporting ongoing deletion requests. We present security definitions that specifically safeguard undeleted data against leakage caused by the deletion of other points. We conceptually separate our focus from the following two types of privacy risks: 

\begin{enumerate}
    \item Our security definition does not address privacy risks that are inherent to the functionality $f$ itself. This mirrors the paradigm of secure Multi-Party Computation (MPC) in cryptography: an MPC protocol guarantees that nothing is leaked beyond the output of the chosen function $f$, but whether evaluating $f$ is acceptable from a privacy standpoint is left as an external, system-level choice. We adopt a similar philosophy, accepting that the initial release of $f(D)$, before any deletion is made, may inherently reveal information about $D$. We instead wish to limit the information subsequently revealed by deleting elements, not taking for granted the leakage that would come from recomputing $f$ after each request.

    \item We separate the risks incurred to {\em deleted} points from the risks incurred to {\em undeleted} points. As we mentioned, the machine unlearning community is already well-aware that there are cases where, by deleting her data, Alice exposes herself to privacy risks (such as in a differencing attack). Nevertheless, Alice should have the right to make a calculated, informed decision to delete her data even in the face of such risks. 
\end{enumerate}

Our point is not that leakage via $f$ itself or w.r.t.\ deleted points is unimportant. Rather, we argue that even if we accept these two types of leakage, there are still other critical risks that need to be addressed. Specifically, our work focuses on the scenario where some users' deletions compromise the privacy of the remaining users who never opted to delete their data. Our goal is to formulate a security definition guaranteeing that Bob's privacy is not harmed by other people deleting their data, beyond what is inherently revealed about Bob's data from the initial release and the value of the deleted points (which the adversary may learn).

\subsection{Non-adaptive attacker}

We start with the simplest variant of our definition that handles {\em non-adaptive} attackers who decide on the sequence of deletions ahead of time. To simplify the presentation, we do not distinguish between the ``algorithm'' and the ``curator'' and consider one stateful algorithm $\AAA$ with the following syntax. Initially $\AAA$ is executed on a dataset $D$ and outputs an initial output $z_0\leftarrow\AAA(D)$. Then, for $k$ rounds $i=1,2,\dots,k$, algorithm $\AAA$ obtains a point $y_i\in D$ and returns a modified output $z_i\leftarrow\AAA(D,y_1,\dots,y_i)$. The mindset is that $z_i$ represents the result of ``deleting'' $y_i$ from $D$, but this is not part of the formal definition. 
We also write $\left(
\AAA(D) , \AAA(D,y_1) , \AAA(D,y_1,y_2) , \dots , \AAA(D,y_1,y_2,\dots,y_k)
\right)$ to denote the vector of outcomes obtained by running $\AAA$ on $D$ and then ``deleting'' one by one the points $y_1,\dots,y_k$.\footnote{We stress that algorithm $\AAA$ can have an internal state, which is omitted from this vector notation for simplicity.}

\begin{definition}[Non-adaptive variant]\label{def:oblivious}
Algorithm $\AAA$ is $k$-{\em undeleted-safe} if there is a simulator $\Sim$ such that for every dataset $D$ of size $n$, and every sequence $(y_1,\dots,y_{k'})$ containing $k'\leq k$ distinct points from $D$ it holds that
$$
\Sim(\AAA(D),y_1,\dots,y_{k'}) \approx 
\left(
\AAA(D) , \AAA(D,y_1) , \AAA(D,y_1,y_2) , \dots , \AAA(D,y_1,y_2,\dots,y_{k'})
\right).
$$
\end{definition}
In words, there is a simulator that takes only the initial output $\AAA(D)$ and the values of the deleted points $(y_1, \dots, y_{k'})$, and can simulate the entire sequence of updated outputs generated by $\AAA$. Intuitively, this formalizes the goal stated above: an attacker observing the entire sequence of unlearning updates learns nothing more about the remaining dataset than what they could already compute themselves from the initial release and the deleted values.

\begin{remark}
Definition~\ref{def:oblivious} can be instantiated with any notion of similarity between the distributions, such as statistical indistinguishability, computational indistinguishability, or the similarity notion used by differential privacy.
\end{remark}

\subsection{Static adaptive attacker}

We now extend the definition to the case where an attacker $\BBB$ can adaptively decide on the {\em order} of deletions. The attacker is static in the sense that the points it controls (and is allowed to delete) are predetermined. We assume that there are no other deletions in the game. In addition, the attacker might possess some side-information about the other points in the dataset $D$, which we denote as $\SideInfo(D)$. Formally, 
for a dataset $D=(x_1,\dots,x_n)$, a set $Y\subseteq [n]$ of size $k'\leq k$ denoting the indices of the points in $D$ controlled by $\BBB$, and a side-information function $\SideInfo$, we consider the following game between algorithm $\AAA$ and the attacker $\BBB$:

\begin{enumerate}
\item[]\hspace{-28px} $\underline{\boldsymbol{{\rm StaticGame}(\AAA,\,D{=}(x_1,\dots,x_n),\,\BBB,\,Y{\subseteq}[n],\,\SideInfo){:}}}$
    \item Let $z_0\leftarrow\AAA(D)$.
    \item The attacker $\BBB$ obtains the controlled points $D\big|_Y\triangleq\{(i,x_i) : i\in Y\}$, the side-information $\SideInfo(D)$, and the initial output $z_0$.
    \item For $\ell=1,2,\dots,|Y|$:
    \begin{enumerate}
        \item The attacker $\BBB$ chooses an index $j_{\ell}\in Y\setminus\{j_1,j_2,\dots,j_{\ell-1}\}$
        \item Algorithm $\AAA$ receives a deletion request for the point $x_{j_{\ell}}$ and returns an updated output $z_{\ell}$, which is given to $\BBB$.
    \end{enumerate}
    \item The attacker $\BBB$ outputs its view of the interaction, which includes its internal randomness, the indices of the controlled points $Y$, its side-information $\SideInfo(D)$, and all the outputs it received from $\AAA$ during the interaction.
\end{enumerate}

\begin{definition}[Static adaptive variant]\label{def:static}
Algorithm $\AAA$ is $k$-{\em undeleted-safe} if for every attacker $\BBB$ there is a simulator $\Sim$ such that for every side-information function $\SideInfo$, every dataset $D$ of size $n$, and every set of indices $Y\subseteq[n]$ of size $|Y|\leq k$ it holds that
$$
\Sim\left(\AAA(D),D\big|_Y,\SideInfo\right) \approx 
{\rm StaticGame}\left(\AAA,D,\BBB,Y,\SideInfo\right).
$$
\end{definition}

Note that any algorithm $\AAA$ that satisfies Definition~\ref{def:static} also satisfies Definition~\ref{def:oblivious}, i.e., Definition~\ref{def:static} is (perhaps) harder to obtain. This holds by taking $\BBB$ to be the attacker that fixes the ordering of the deletions from $Y$ before the interaction begins, and setting $\SideInfo$ to be empty.

\subsection{Dynamic adaptive attacker}

In the previous two definitions, the set of indices/points that the attacker controlled was fixed ahead of time. We now handle the case where the attacker can dynamically choose the indices to control and delete. So now $Y$, the set of indices controlled by the attacker, is not fixed in advance. It is initially empty, and throughout the execution, the attacker dynamically decides on indices from $[n]$ to ``corrupt''. Once corrupted, these indices are added to $Y$. 
Formally, given a dataset $D=(x_1,\dots,x_n)$ we consider the following game between a stateful algorithm $\AAA$ and an attacker $\hat{\BBB}$ with side-information function $\SideInfo$. This game has two modes, parameterized by the bit $b\in\{0,1\}$, where $b=0$ corresponds to a ``real world'' execution in which the attacker learns the sequence of modified outputs of $\AAA$ after every deletion, and where $b=1$ corresponds to an ``ideal world'' execution in which this information is hidden from the attacker.

\begin{enumerate}
\item[]\hspace{-28px} $\underline{\boldsymbol{{\rm DynamicGame}_b(\AAA,\,\hat{\BBB},\,D{=}(x_1,\dots,x_n),\,\SideInfo){:}}}$
    \item Let $\ell=0$ and $z_0\leftarrow\AAA(D)$.
    \item The attacker $\hat{\BBB}$ obtains $\SideInfo(D)$ and the initial output $z_0$.
    \item Initiate $Y=R=\emptyset$. \gray{\small \% Here $Y$ denotes the set of indices controlled by the attacker, and $R$ denotes the indices it had already deleted.}
    \item While $|R|<k$ and the attacker $\hat{\BBB}$ has not ended the loop: %
    \begin{enumerate}
    \item Set $\ell\leftarrow\ell+1$.
        \item The attacker $\hat{\BBB}$ chooses a set of indices $I_{\ell}\subseteq[n]$, possibly empty, under the restriction that $|Y\cup I_{\ell}|\leq k$, and learns $x_i$ for every $i\in I_{\ell}$. Set $Y\leftarrow Y\cup I_{\ell}$.
        \item The attacker $\hat{\BBB}$ chooses a set of indices $J_{\ell}\subseteq Y\setminus R$, possibly empty. Set $R\leftarrow R\cup J_{\ell}$.
        \item If $b=0$ and $J_{\ell}\neq\emptyset$ then $\AAA$ receives a deletion request for the points in $J_{\ell}$ and returns an updated output $z_{\ell}$, which is given to the attacker. If $b=1$, do nothing. %

        \item The attacker $\hat{\BBB}$ chooses whether to continue or to end the loop.
    \end{enumerate}
    \item The attacker $\hat{\BBB}$ outputs $v$.  \gray{\small \% In a ``real world'' execution, where $b=0$, this $v$ can be thought of as the view of the attacker $\hat{\BBB}$ at the end of the execution. However, in an ``ideal world'' execution, where $b=1$, we think of $\hat{\BBB}$ as a {\em simulator} that aims to simulate some other attacker, and so the value $v$ is not necessarily the view of $\hat{\BBB}$ itself, but rather a ``simulated view'' that should correspond to the view of the simulated attacker at the end of a real world execution.} 

    \item The output of the game is $(Y,v)$. \gray{\small \% We include $Y$ in the output of the game to force the simulator (when $b=1$) to choose its corruptions similarly to the real world attacker (when $b=0$). Otherwise, for example, the simulator could corrupt $k$ indices and use them to simulate a real-world attacker that only corrupts $k/2$ indices.} %
\end{enumerate}

\begin{definition}[Dynamic adaptive variant]\label{def:dynamic}
Algorithm $\AAA$ is $k$-{\em undeleted-safe} if for every attacker $\BBB$ there is a simulator $\Sim$ such that for every side-information function $\SideInfo$ and every dataset $D$ of size $n$ it holds that
$$
{\rm DynamicGame}_1(\AAA,\Sim,D,\SideInfo)
\approx 
{\rm DynamicGame}_0(\AAA,\BBB,D,\SideInfo).
$$
\end{definition}

Note that any algorithm $\AAA$ that satisfies Definition~\ref{def:dynamic} also satisfies Definition~\ref{def:static}, i.e., Definition~\ref{def:dynamic} is (perhaps) harder to obtain. This holds by taking $\BBB$ to be the attacker that fixes the set $Y$ before the interaction begins.

\section{Algorithms satisfying our definition and extensions}

It is immediate from Definition~\ref{def:dynamic} that it is satisfied by any {\em stateless} algorithm $\AAA$, where by {\em stateless} we mean that for any $D$ and deleted points $y_1,\dots,y_k$ it holds that  $\left(
\AAA(D) , \AAA(D,y_1) , \dots , \AAA(D,y_1,y_2,\dots,y_k)
\right)$ is strictly determined by  $\left(
\AAA(D) , y_1,\dots,y_k\right)$. This includes interesting functionalities like a bulletin board and exact summations. Note that $\AAA(D)$ does not have to be deterministic for this argument to go through. 
\begin{example}\label{ex:exactSum}
Consider the algorithm $\AAA$ that takes a dataset containing $n$ numbers from $[0,1]$ and returns a DP estimation for their summation. Then, upon receiving a deletion request for a point $x_i$, the algorithm simply returns $(z-x_i)$ where $z$ is the previously released sum. This algorithm is stateless, so it satisfies our new definition, and it provides DP w.r.t.\ undeleted points.    
\end{example}
Note that this task can also be solved in the continual observation model using the classical ``tree algorithm'' of \citet{DworkNPR10}, which would guarantee DP for both the deleted and the undeleted points. However, this provably requires error that scales with $\Omega(\log k)$ \citep{Cohen0NSS24}. The simple algorithm described in Example~\ref{ex:exactSum} shows that if we only want to provide privacy for undeleted points, then a single DP calculation (the initial noisy summation) suffices, resulting in error independent of $k$.

It is also immediate from our definition that some algorithms/functionalities do {\em not} satisfy it. For example, as we mentioned, the exact median allows the attacker to learn undeleted points, and thus does not satisfy our new definition.\footnote{To see this, consider a dataset $D$ containing $10$ copies of the point 0, one ``middle'' point $w$ which is either $w=1/4$ or $w=3/4$, and $12$ copies of the point 1. The exact median is 1. Suppose that the attacker controls one of the copies of $1$ (so $k=1$ in this example). By deleting this controlled point, the exact median becomes $w$, which is then revealed to the attacker. But without deletions the simulator does not get any information about $w$ and hence could not conduct the simulation.} 
So suppose that we are interested in computing such function $f$ that is not directly compatible with our framework. The question then becomes: {\em ``how much additional information is needed, beyond $f$, in order for the simulation to become possible?''}. For example, observe that if $\SideInfo$ contains all of $D$ (the side-information available to the simulator and the attacker), then simulating the exact median, or any other functionality, becomes trivial. We aim to quantify this ``additional leakage'' that suffices for simulation, and minimize it. In particular, in some cases, it might be reasonable to allow for additional leakage that satisfies DP. We can argue about the additional leakage via two primary approaches:
\begin{enumerate}
    \item[(1)] We could explicitly give the simulator some additional leakage function $g$ of the dataset $D$, in order for the simulation to become possible. This has the benefit of making the leakage explicit; hopefully allowing us to argue about it more easily and try to minimize it.
    
    \item[(2)] Instead (or maybe in addition) to considering such a leakage function $g$, it might also be useful to replace $f$ with a {\em different} function $\hat{f}$ that approximates it in some sense while being more suited to our definitions.\footnote{This is common practice in the literature on differential privacy: When designing private analogous to non-private computations, we typically compute noisy/approximate variants of the original functionality.} Computing $\hat{f}$ instead of $f$ might require less, or no, additional leakage (in the form of $g$). However, this might introduce another type of leakage because in our definition we do not protect against leakage caused by the initial computation, which would now be according to $\hat{f}$ instead of $f$.
\end{enumerate}

We formulate option (1) in the following definition.\footnote{For concreteness, we state this variant of the definition only w.r.t.\ Definition~\ref{def:dynamic}. The same can be said w.r.t.\ Definitions~\ref{def:static} and~\ref{def:oblivious}.}

\begin{definition}[Dynamic adaptive variant, with leakage]\label{def:dynamicLeakage}
Let $g$ be a leakage function. Algorithm $\AAA$ is $(k,g)$-{\em undeleted-safe with leakage} if for every attacker $\BBB$ there is a simulator $\Sim(\cdot)$ that takes an additional input and then plays in ${\rm DynamicGame}_1$ such that the following holds. For every side-information function $\SideInfo$ and every dataset $D$ of size $n$ it holds that
$$
{\rm DynamicGame}_1(\AAA,\Sim(g(D)),D,\SideInfo)
\approx 
{\rm DynamicGame}_0(\AAA,\BBB,D,\SideInfo).
$$
\end{definition}

Note that any algorithm $\AAA$ that is $k$-undeleted-safe according to Definition~\ref{def:dynamic} is also $(k,g)$-undeleted-safe according to Definition~\ref{def:dynamicLeakage} with $g\equiv\bot$. On the other extreme, for $g(D)=D$ Definition~\ref{def:dynamicLeakage} becomes trivial in the sense that {\em any} algorithm $\AAA$ would be $(k,g)$-undeleted-safe (this holds because if the simulator knows $D$ then it can simulate all of the execution of ${\rm DynamicGame}_0$). When $g$ is ``bounded'', say a DP function, then the definition might be considered reasonable.

\subsection{Approximate median}

Suppose that we are interested in an algorithm that initially takes a dataset $D$ containing $n$ elements from a domain $X$ and releases an approximation $z_0$ to the median of $D$. Then, for $k$ rounds, the algorithm receives the $i$-th deletion request and outputs a modified answer $z_i$. We evaluate the approximation of the median for the (remaining) data at any given moment using its {\em rank error}. That is, if the current dataset is $D_i$ and the algorithm returns an answer $z_i$, then the error of $z_i$ w.r.t.\ $D_i$ is defined as
$$
\error_{D_i}(z_i):=\left|\{ x\in D_i : x \text{ is strictly between } z_i \text{ and } \mbox{median}(D_i) \}\right|.
$$
We aim to design an algorithm for approximating the median, while satisfying our security definition, that guarantee small worst-case error across all their $(k+1)$ releases. 
Naively, one could solve this task by first computing and releasing the exact median, and then ignoring all deletion requests. This stateless algorithm satisfies our security definition and guarantees error at most $k$ deterministically. 
Alternatively, we could extend Example~\ref{ex:exactSum} to this setting as follows.

\begin{example}\label{ex:med}
Consider the algorithm $\AAA$ that takes a dataset $D$ containing $n$ points from some ordered domain $X$ and operates as follows. Initially, $\AAA$ computes a differentially private data structure $P$ that provides approximate answers to range queries w.r.t.\ $D$. Specifically, for every $a\leq b\in X$ it ensures that $P(a,b)\approx|\{x\in D : a\leq x\leq b\}|$.\footnote{There are several such constructions in the literature, e.g., by \citet{BeimelNS13,BunNSV15,KaplanLMNS20,Cohen0NSS23}.} Note that this, in particular, allows us to compute an approximate median of the dataset $D$. Now, throughout the execution, we simply subtract deleted elements from the approximate counts given by $P$. Specifically, at any moment throughout the execution, we estimate the number of (remaining) elements in a range $[a,b]$ as $P(a,b)-\#_{\rm del}(a,b)$, where $\#_{\rm del}(a,b)$ denotes the (exact) number of deleted elements between $a$ and $b$. As in Example~\ref{ex:exactSum}, this provides DP w.r.t.\ undeleted points, without increasing the error introduced by $P$.
\end{example}

\subsection{A recipe for undeleted-safe algorithms}

Example~\ref{ex:med} suggests the following informal recipe for fitting a functionality $f$ to our framework:
\begin{enumerate}
    \item[(1)] Identify a set of sufficient statistics for the function $f$.
    \item[(2)] Estimate these statistics with differential privacy.
    \item[(3)] Throughout the execution, maintain these statistics as a function of the deleted points.
\end{enumerate}

Of course, this recipe is only useful when the set of sufficient statistics we identify is compatible with differential privacy. It turns out that this is indeed the case for some functionalities that were studied by prior work on machine unlearning. 

In particular, \citet{cao2015towards} presented the following ``SQ framework'' for machine unlearning: Let $f$ be any function which can be computed in the non-adaptive {\em statistical query (SQ)} model. That is, there exists a collection of predicates $C$ such that for any dataset $D$, the value of $f(D)$ can be computed by post-processing $\{\sum_{x\in D}c(x)\}_{c\in C}$, i.e., the sum of each of the predicates in $C$ on the elements in $D$. As \citet{cao2015towards} showed, tracking $\{\sum_{x\in D}c(x)\}_{c\in C}$ throughout the execution allows us to maintain the current value of $f$ without recomputing it from scratch. When $|C|$ is relatively small, this results in an efficient algorithm for perfect retraining. This SQ framework is particularly suited for our definition, as computing the initial values of $\{\sum_{x\in D}c(x)\}_{c\in C}$ can be done in a DP manner (again, assuming that $|C|$ is small). Afterwards, as in Examples~\ref{ex:exactSum} and~\ref{ex:med}, we could update these values as a function of the deleted points alone.

\bibliographystyle{plainnat}

\appendix

\section{$k$-means clustering on the line}
\label{sec:kMeansClustering}

Consider a perfect retraining algorithm for $k$-means that, upon deletion, simply (re)runs Lloyd’s algorithm \citep{lloyd1982least} on the remaining data points.\footnote{ 
Recall that Lloyd's algorithm for $k$-means is an iterative refinement procedure that alternates between assigning data points to the nearest cluster center and updating the cluster centers to be the mean of their assigned points. These iterative refinement procedure continues until convergence (when assignments no longer change).}
Furthermore, for the sake of this example, let us consider the simple case where $k=2$ and the data points are on the real line. 

We consider the following experiment. Initially, a dataset $D$ containing $n$ points is sampled from some fixed (known) distribution $\DDD$. Then, $\alpha n$ random points from the dataset $D$ are given to an attacker $\AAA$. We refer to these points as the points that the attacker {\em controls}, and denote them as $A\subseteq D$. The game now continues in iteration where in every iteration:
\begin{enumerate}
    \item The attacker $\AAA$ picks a point $a\in A\cap D$ and requests it to be deleted from $D$.
    \item We remove $a$ from $D$, rerun Lloyd's algorithm on $D$ to obtain a new tuple $(c'_1,c'_2)$ and give this tuple to $\AAA$.
\end{enumerate}

The goal of the attacker is to identify points which it does not control, i.e., discover points from $D\setminus A$. We show empirically that this can be done successfully for several choices for the underlying distribution $\DDD$, including the uniform distribution over an interval, and a mixture of two normal distributions. 

The intuition behind the attack is as follows. Suppose that the attacker knows, in addition to the current cluster centers $(c_1,c_2)$ also knew the {\em sizes} of the clusters $(m_1,m_2)$, where $m_1+m_2=|D|$. By deleting a point from $A$ which is larger than $c_2$, the attacker causes the right cluster to shift slightly to the left, thereby potentially causing points from the left cluster to ``jump'' to the right cluster. If it happens that exactly one point jumps between the clusters, then the attacker learns this point exactly, as it knows the previous and new averages (i.e., centers) of the left cluster (and we assumed that the attacker knows the cluster sizes). We show empirically that this indeed happens several times throughout the execution.\footnote{If more than one point jumps between the clusters, say $r$ points, then the attacker could heuristically delete a point from the other side of the line, thereby causing $j$ points to potentially jump in the other direction. The attack now works if $j\in\{1,r-1,r+1\}$ since by looking at two consecutive jumps of $r,r+1$ points the attacker could isolate a single point.}

It remains, therefore, to explain how the attacker can learn the sizes of the clusters from the cluster centers. Suppose that by deleting a point $a^*$ from the right side of the line, the attacker caused both centers to change, from $(c_1^{\rm old},c_2^{\rm old})$ to $(c_1^{\rm new},c_2^{\rm new})$. Let $m_1^{\rm old},m_2^{\rm old},m_1^{\rm new},m_2^{\rm new}$ denote the sizes of the corresponding clusters before and after the change. Let $P$ denote the sum of the points that jumped from the left to the right cluster and let $p=m_1^{\rm old}-m_1^{\rm new}=m_2^{\rm new}-m_2^{\rm old}+1$ denote the number jumping points, where the $(+1)$ is because of the deleted point $a^*$. Also let $n=m_1^{\rm old}+m_2^{\rm old}=m_1^{\rm new}+m_2^{\rm new}+1$ denote the size of the dataset before the deletion, which is known to the attacker. We have that
$$
P=m_1^{\rm old}\cdot c_1^{\rm old} - m_1^{\rm new}\cdot c_1^{\rm new}
=m_1^{\rm old}\cdot c_1^{\rm old} - \left(m_1^{\rm old}-p\right)\cdot c_1^{\rm new}.
$$
Now, the sum of the points in the right cluster, after the change, can be written in two ways: Either 
\begin{equation}\label{eq:twoways1}
m_2^{\rm new}\cdot c_2^{\rm new}=\left(m_2^{\rm old} -1 + p\right)\cdot c_2^{\rm new}=\left(n-m_1^{\rm old} -1 + p\right)\cdot c_2^{\rm new},
\end{equation}
or it can be written as
\begin{equation}\label{eq:twoways2}
m_2^{\rm old}\cdot c_2^{\rm old} + P - a^*
= \left(n-m_1^{\rm old}\right)\cdot c_2^{\rm old} + m_1^{\rm old}\cdot c_1^{\rm old} - \left(m_1^{\rm old}-p\right)\cdot c_1^{\rm new}-a^*.
\end{equation}
Equating (\ref{eq:twoways1}) and (\ref{eq:twoways2}) we get one equation with two unknown variables, $m_1^{\rm old}$ and $p$. 
Isolating $p$ from this equation allows us to express $p$ as a function of $m_1^{\rm old}$ as follows:
\begin{align*}
p
&=\frac{m_1^{\rm old} \cdot\left(c_1^{\rm old}-c_1^{\rm new}-c_2^{\rm old}+c_2^{\rm new}
\right)   
+n\cdot(c_2^{\rm old} - c_2^{\rm new})
+ c_2^{\rm new}  
-a^* }{c_2^{\rm new}-c_1^{\rm new}}
\end{align*}
Now recall that both $p$ and $m_1^{\rm old}$ must be integers. Heuristically, by trying different values for $m_1^{\rm old}$ until the resulting $p$ is an integer, we succeed in fining the correct values for $p$ and $m_1^{\rm old}$, and thus learn the sizes of the clusters exactly. The code for this simulation is available on \href{https://github.com/refael-kohen/k_means_attack}{GitHub}.

\section{An attacker controlling ${\boldsymbol{\omega(1)}}$ points}\label{sec:weakAttacker}

In this section we give the formal details for an attack scenario in which the attacker controls (and deletes) merely $\omega(1)$ points and is still able to reconstruct $\approx n$ ``uncontrolled'' (or ``undeleted'') points. 

\begin{definition}[The batch queries (BQ) problem]\label{def:BQ}
Let $n,k,t\in\N$ and $\alpha,\beta\in[0,1]$ be parameters, and let $Q=(q_1,\dots,q_n)$ be a collection of $n$ counting queries over $[n]$. That is, every $q$ is a predicate $q:[n]\rightarrow\{0,1\}$.
Let $T\subseteq[n]\cup\{\star\}$ be a multiset and let $\vec{a}\in\R^t$ be a real valued vector. 
We say that $\vec{a}$ is $\alpha$-accurate for $T$ if there exists an integer $i$ such that
$$
\left\lceil \frac{{\rm Star}(T)-k}{3k} \right\rceil \leq i \leq \left\lfloor \frac{{\rm Star}(T)+k}{3k} \right\rfloor,
$$
and
$$
\frac{1}{t}\left\|\vec{a} - \left(q_{(i-1)t+1}(T\setminus\{\star\}),...,q_{it}(T\setminus\{\star\})\right)\right\|^2_{2}
\leq\alpha.
$$
We consider two types of algorithms:
\begin{itemize}
    \item {\bf The one shot setting.} Let $\AAA$ be an algorithm that takes a multiset $T\subseteq[n]\cup\{\star\}$ and returns a vector in $\R^t$. We say that this algorithm $(\alpha,\beta)$-solves the one shot variant of the ${\rm BQ}_{Q,n,k,t}$ problem if for every $T$, with probability at least $1-\beta$ it holds that $\AAA(T)$ is $\alpha$-accurate w.r.t.\ $T$.

    \item {\bf The continual setting.} Let $r\in\N$ be a parameter. Consider an algorithm $\AAA$ that is initially instantiated with a multiset $T\subseteq[n]\cup\{\star\}$ and returns a vector $\vec{a}_0\in\R^t$. Then, in every round $\ell=1,2,\dots,r$, algorithm $\AAA$ obtains a ``deletion request'' $z_{\ell}\in[n]\cup\{\star\}$ and returns an updated vector $\vec{a}_{\ell}\in\R^t$. We say that this algorithm $(\alpha,\beta)$-solves the ${\rm BQ}_{Q,n,k,t}$ problem over $r$ rounds in the continual setting if for every $T$ and every adaptive strategy for selecting the $z_{\ell}$'s, with probability at least $1-\beta$, for every $\ell$ it holds that $\vec{a}_{\ell}$ is $\alpha$-accurate w.r.t.\ $T\setminus\{z_1,\dots,z_{\ell}\}$.
\end{itemize}
\end{definition}

\begin{theorem}\label{thm:DPoneshot}
Let $n,k,t\in\N$ and $\beta,\eps,\delta\in[0,1]$ be such that 
$$
k\geq\frac{2}{\eps}\log\frac{2}{\beta}
\qquad\text{and}\qquad
t=o\left(\frac{\eps^2 n}{\log(1/\delta)}\right)
\qquad\text{and}\qquad
n=\omega\left(\frac{\log(1/\delta)\log(1/\beta)}{\eps^2}\right).
$$
Then, for any set $Q$ of counting queries, there exists an $(\eps,\delta)$-DP algorithm for $(\alpha,\beta)$-solving the ${\rm BQ}_{Q,n,k,t}$ problem in the one-shot setting for $\alpha=o(n)$.
\end{theorem}

\begin{proof}
The construction is straightforward via the Laplace or Gaussian mechanism. Specifically, given the multiset $T$, we first use the Laplace mechanism to find an appropriate index $i$. That is, we compute:
$$
\hat{i}\leftarrow \frac{{\rm Star}(T) + \Lap(\frac{2}{\eps})}{3k}.
$$
This step satisfies $(\frac{\eps}{2},0)$-DP. Furthermore, the chosen index $\hat{i}$ is valid (according to Definition~\ref{def:BQ}) whenever $|\Lap(\frac{2}{\eps})|\leq k$, which happens with probability at least $1-\beta/2$ provided that $k\geq\frac{2}{\eps}\log\frac{2}{\beta}$. 

Given $\hat{i}$, all that remains is to answer to corresponding set of $t$ queries from $Q$, which we again do via the Gaussian mechanism. Specifically, we answer each of these $t$ queries with independent noise sampled from $\mathcal{N}(0,\sigma^2)$ for $\sigma=\frac{ \sqrt{16 t\log(1/\delta)}}{\eps}$. By composition theorems for DP, this satisfies $(\frac{\eps}{2},\delta)$-DP, and thus (again by composition), the overall algorithm satisfies $(\eps,\delta)$-DP as required.

Furthermore, by standard tail bounds for sums of independent squared normal random variables (aka the chi squared distribution), with probability at least $1-\beta/2$, our average squared errors (denoted as $\alpha$) satisfies 
$$
\alpha\leq \frac{\sigma^2 \left( t + 2\sqrt{t\ln(2/\beta)}+2\ln(2/\beta)\right)}{t}
=
\frac{16\log(1/\delta)\left( t + 2\sqrt{t\ln(2/\beta)}+2\ln(2/\beta)\right)}{\eps^2},
$$
which is $o(n)$ provided that
$$
t=o\left(\frac{\eps^2 n}{\log(1/\delta)}\right) \qquad\text{and}\qquad n=\omega\left(\frac{\log(1/\delta)\log(1/\beta)}{\eps^2}\right),
$$
as required.
\end{proof}

\begin{theorem}\label{thm:applyDY}
Let $n,k,t\in\N$ and $\alpha,\beta\in[0,1]$ be such that $\alpha\leq o(n)$. There exists a set of $n$ counting queries $Q$ and there exists an attacker $\BBB$ such that the following holds.
For any algorithm $\AAA$ for $(\alpha,\beta)$-solving the ${\rm BQ}_{Q,n,k,t}$ problem in the continual setting over $r=\lceil n/t \rceil$ rounds, and for any set $D\subseteq[n]$, if $\AAA$ is instantiated on the multiset $T$ containing $D$ as well as $3kr$ copies of $\star$, and then interacts with $\BBB$ over $r$ rounds (as in Definition~\ref{def:BQ}), then with probability at least $1-\beta$, at the end of the interaction the attacker $\BBB$ outputs a set $\hat{D}$ such that $\left|D\triangle\hat{D}\right|\leq \alpha$.
\end{theorem}

\begin{corollary}
Combining Theorems~\ref{thm:DPoneshot} and~\ref{thm:applyDY}, 
for $n,k,t\in\N$ satisfying $k\geq\omega(1)$, and $n=\omega\left( 1 \right)$, and $t=o\left( n \right)$, 
we get that
\begin{enumerate}
    \item The ${\rm BQ}_{Q,n,k,t}$ problem can be $\left(o(n),\beta\right)$-solved in the one shot setting while satisfying $\left(o(1),o(1)\right)$-DP for any constant $\beta$.

    \item Any algorithm for $\left(o(n),\beta\right)$-solving the ${\rm BQ}_{Q,n,k,t}$ problem in the continual setting allows an attacker controlling merely $\omega(1)$ points to reconstruct almost all of the other points from the dataset with probability at least $1-\beta$.
\end{enumerate}
\end{corollary}

\begin{proof}[Proof of Theorem~\ref{thm:applyDY}]
The construction follows from the results of
\cite{dwork2008new}, who presented a set of $n$ counting queries $Q$ and an attacker $\BBB_{\rm DY}$ such that for every set $D\subseteq[n]$, 
given approximate answers $a_q\approx q(D)$ for every $q$ in $Q$ satisfying
$$
\frac{1}{n}\sum_{q\in Q} (a_q - q(D))^2\leq o(n),
$$
the attacker $\BBB_{\rm DY}$ can efficiently reconstruct a set $\hat{D}$ with $\left|D\triangle\hat{D}\right|\leq o(n)$.

We implement this attacker in our setting as follows. Let $T$ be the multiset containing $D$ as well as $3kr$ copies of $\star$. Now consider an attacker $\BBB$ that plays against an algorithm $\AAA$ in the continual batch queries problem. Initially, $\AAA$ is instantiated with $T$. Then, in every round, the attacker $\BBB$ deletes $3k$ copies of $\star$. By the definition of the BQ problem, this allows $\BBB$ to obtain approximate answers $a_q$ to every $q\in Q$ such that with probability at least $1-\beta$ we have
$$
\frac{1}{n}\sum_{q\in Q} (a_q - q(D))^2\leq o(n).
$$
This allows $\BBB$ to run $\BBB_{\rm DY}$ to obtain a suitable reconstruction $\hat{D}$.
\end{proof}

\end{document}